\documentclass[letterpaper]{article} 
\usepackage{aaai21}  
\usepackage{times}  
\usepackage{helvet} 
\usepackage{courier}  
\usepackage[hyphens]{url}  
\usepackage{graphicx} 
\usepackage{natbib}  
\usepackage{caption} 
\usepackage{algorithmic}
\usepackage{appendix}
\usepackage[ruled, linesnumbered,vlined]{algorithm2e}
\usepackage{amsmath}
\usepackage{amssymb}
\usepackage{amsthm}
\usepackage{array}
\usepackage{bbding}
\usepackage{booktabs}
\usepackage{color}
\usepackage{epstopdf}
\usepackage{float}
\usepackage{graphicx}
\usepackage{subfigure}
\usepackage{multicol}
\usepackage{multirow}
\usepackage{wrapfig}
\usepackage{url}
\usepackage{indentfirst}
\usepackage{delarray}
\usepackage{tabularx}
\usepackage{textcomp}
\usepackage{mathrsfs}
\usepackage{bm}
\usepackage{setspace}
\usepackage{makecell}
\usepackage{pifont}
\usepackage{tikz}
\newcommand*{\circled}[1]{\lower.7ex\hbox{\tikz\draw (0pt, 0pt)%
		circle (.5em) node {\makebox[1em][c]{\small #1}};}}
\usepackage{textcomp}
\usepackage[T1]{fontenc} 
\usepackage{fancyvrb,cprotect}
\newtheorem{theorem}{Theorem}
\usepackage{hyperref}
\makeatletter
\def\hlinew#1{
	\noalign{\ifnum0=`}\fi\hrule \@height #1 \futurelet
	\reserved@a\@xhline}
\makeatother

\title{ Against Adversarial Learning: Naturally Distinguish Known and Unknown in Open Set Domain Adaptation }

\author{
	Sitong Mao \textsuperscript{\rm 1}, 
	Xiao Shen \textsuperscript{\rm 1}, 
	Fu-lai Chung \textsuperscript{\rm 1} \\
}

\affiliations{
	\textsuperscript{\rm 1} The Hong Kong Polytechnic University \\
	sitong.mao@connect.polyu.hk, xiao.shen@connect.polyu.hk, korris.chung@polyu.edu.hk
}

\begin{document}
\maketitle
\begin{abstract}
Open set domain adaptation refers to the scenario that the target domain contains categories that do not exist in the source domain. It is a more common situation in the reality compared with the typical closed set domain adaptation where the source domain and the target domain contain the same categories. The main difficulty of open set domain adaptation is that we need to distinguish which target data belongs to the unknown classes when machine learning models only have concepts about what they know. In this paper, we propose an ``against adversarial learning'' method that can distinguish unknown target data and known data naturally without setting any additional hyper parameters and the target data predicted to the known classes can be classified at the same time. Experimental results show that the proposed method can make significant improvement in performance compared with several state-of-the-art methods. 
\end{abstract}		

\section{Introduction}
Domain adaptation refers to the learning scenario that adapts a model to the unlabeled or a few labeled target data by exploiting the information of the labeled source data which is from different but related domain~\cite{farseev2017cross}~\cite{liu2008evigan}~\cite{mcclosky2006reranking}~\cite{daume2009frustratingly}~\cite{saenko2010adapting}~\cite{long2015learning}~\cite{long2016deep}~\cite{gong2012geodesic}~\cite{weston2012deep}~\cite{pan2011domain}~\cite{tzeng2014deep}. Due to domain shift, the model trained on source data cannot work well if directly applied on the target data. Therefore, it is important to develop domain adaptation methods to enhance the performance. Deep neural network has been proved having good transferability, especially the first few layers~\cite{yosinski2014transferable}. This inspires developing methods that combine domain adaptation with deep neural networks by reducing the distribution discrepancy between the deep features of the source data and the target data. One way to do so is to minimize a statistical metric that can reflect the distribution discrepancy~\cite{long2015learning}~\cite{long2017deep}~\cite{mao2018deep}.

Recently, adversarial learning~\cite{goodfellow2014generative} has been used for domain adaptation and achieved outstanding performance. The architecture of adversarial domain adaptation mainly consists of two parts: feature extractor and discriminator. The discriminator aims at distinguishing the target data from the source data, while the feature extractor tries to extract deep features that can confuse the discriminator. Theoretically, its global optimum\footnote{Proof is available in Appendix A.} achieves when deep features of the source data and deep features of the target data have the same distribution, so that the deep model trained on the source data can adapt to the target data. Some adversarial domain adaptation methods map source deep features and target deep features close without considering their label information~\cite{ganin2016domain}~\cite{tzeng2017adversarial}. To further improve the performance, tensor products of the features and the probabilities distributed over each category are used as the inputs of the discriminator~\cite{long2018conditional} to reduce the conditional distribution discrepancy \cite{zhang2019bridging} \cite{zhang2019domain}.

In the previous standard settings of the closed set domain adaptation, source data and target data consist of the same set of categories. However, this is not always true in the reality. In this paper, we focus on a more common situation proposed recently, i.e., open set domain adaptation~\cite{panareda2017open}~\cite{saito2018open}~\cite{liu2019separate}~\cite{lian2019known}~\cite{luo2017label}. Open set domain adaptation (OSDA) does not require source domain and target domain containing the same set of categories. Instead, the source domain and the target domain have a set of shared categories and a set of domain-specific categories. In this paper, we focus on the situation that the target domain includes all the categories of the source domain~\cite{saito2018open} and contains target-specific categories at the same time.


The difference of the OSDA compared with the standard closed set domain adaptation is that two tasks should be realized at the same time: 1) Reducing the distribution discrepancy of the shared classes between source domain and target domain; 2) Distinguishing the unknown target data from data belonging to the known categories. The main challenge is that when there exist target categories which are not included in the source domain, we cannot get any information of the boundary between the known and the unknown data from the source domain. Thus, previous standard domain adaptation methods which reduce the distribution discrepancy between all the source data and all the target data will cause negative transfer because of mapping features of unknown target data to known categories. To address this problem, several methods have been developed for the OSDA recently \cite{panareda2017open} \cite{saito2018open} \cite{liu2019separate}. However, in these methods, the boundary between the known data and the unknown data needs to be defined manually by using a threshold or relying on splitting the known source data to find the boundary.

This paper proposes to naturally classify known and unknown target data in OSDA by training against the adversarial learning (AAL). In the proposed method, we do not need to define the boundary between ``known'' and ``unknown'' by a preset threshold or manually define known source data and unknown source data. Instead, the proposed method treats all target data as unknown at first following the fact that all target labels are unavailable during the training process. A classifier is used to classify all target data to the unknown class, which is against the adversarial training process. The adversarial learning maps target data close to source data from different degrees according to the probability of the target sample belonging to the shared classes. Along with the training process, unknown target data can be distinguished and known target data can be classified simultaneously. The details of the proposed method will be given in Section 3. 

The main contributions of this paper are: 1) We propose a method that trains against adversarial learning which can naturally distinguish target data belonging to known classes and unknown classes without any manually defined thresholds or hyper parameters; 2) A dynamic classifier and a weighted adversarial training process are built to map the deep features properly; 3) The  proposed AAL achieves a new state-of-the-art level with remarkable performance on all $18$ tasks of two datasets.
   
\section{Related Work}

Different from closed set domain adaptation methods, it is undesirable to map all target data and all source data close in OSDA since there exists target data that does not belong to any category in the source domain. Thus, it is important to distinguish the unknown target data in OSDA. Some approaches label those target instances that potentially belong to the categories of interest presented in the source dataset through optimizing a statistical function. For example, Assign-and-Transform-Iteratively (ATI)~\cite{panareda2017open} maps the feature space of the source domain close to the target domain according to the distance between each target data and the mean values of all samples in each class of the source domain. Consequently, whether a target sample belongs to one of the known classes or the unknown class is predicted. In other methods that utilize adversarial training, a threshold is used. By exploiting adversarial learning, OSDA through Back Propagation (OSBP)~\cite{saito2018open} pre-defines a threshold for the generator to reject a target data as unknown or map it close to the known source data. Then in ``Separate to adapt'' (STA)~\cite{liu2019separate}, binary classifiers are trained on each category of the source data trying to make the network have a concept of boundary between known and unknown and then training another binary classifier on target data according to their similarity to source data. All these previous OSDA works need to manually define a threshold or relabeling the source data and building a set of additional classifiers to help find the boundary between known and unknown. 

In this paper, the proposed AAL directly trains a classifier with all the target data labeled as unknown at the beginning, which follows the fact that the labels of all target data are not available during the training time. By training against the adversarial learning architecture, the proposed method can assign the target data to known or unknown class naturally without any additional hyper parameters to artificially define the boundary between known data and unknown target data.

\begin{figure*}
	\centering
	\includegraphics[width=0.80\linewidth, height=0.32\linewidth]{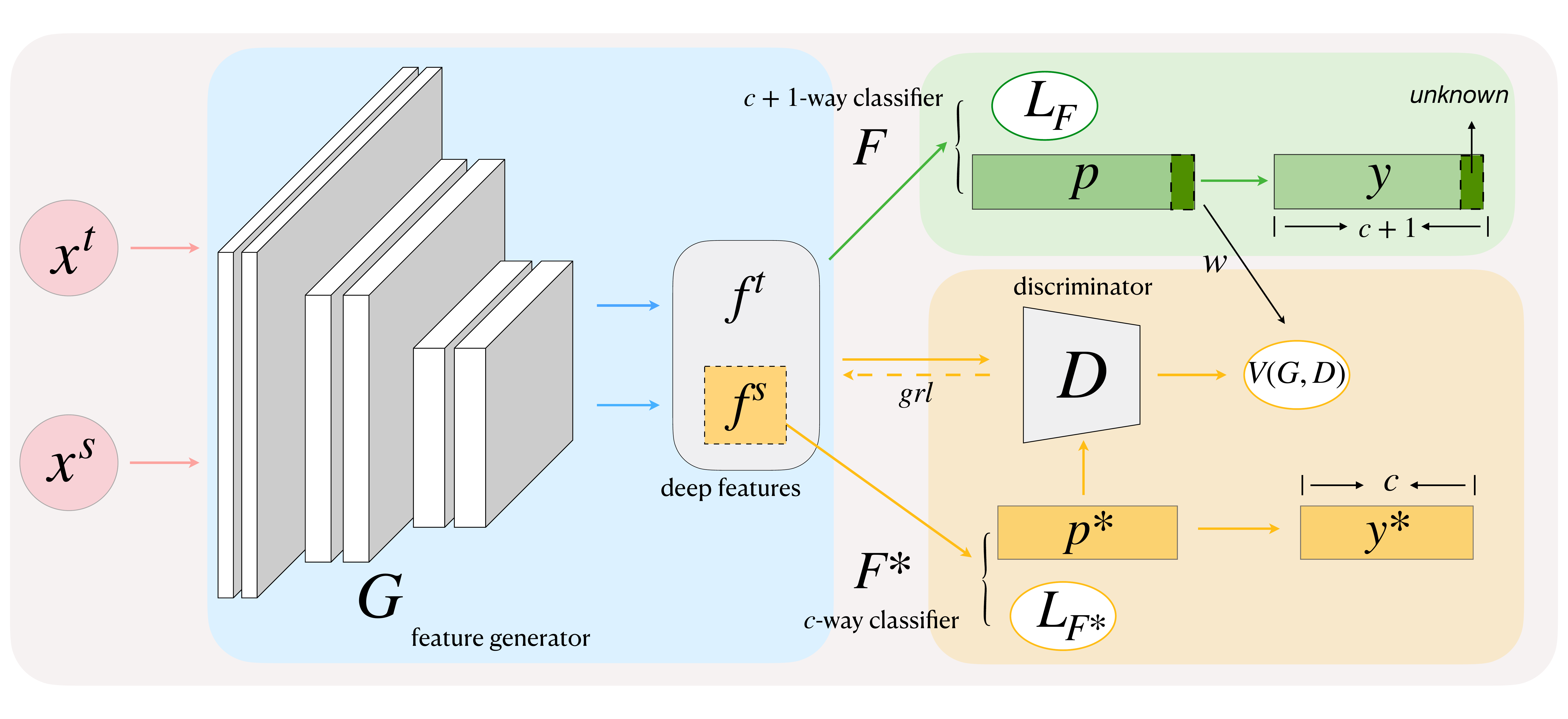} 
	\caption[LoF entry]{Illustration of the proposed OSDA method. $x^s$ denotes the source data and $x^t$ denotes the target data. $f^s$ and $f^t$ are the deep features of source data and target data extracted by the deep neural network $G$ respectively. $D$ is the discriminator used to distinguish which domain are the deep features belongs to. $L_d$ is the loss of $D$. Gradient reversal layer ($grl$) is used to back-propagate the gradients of $D$ to $G$~\cite{ganin2016domain}. $y^*=\{1,...,c\}$ is the label of the source data. $F^*$ denotes the $c$-way classifier which aims to classify the source data. Function $F$ is a $c+1$-way classifier which classifies the source data to their corresponding classes and all target data to class $c+1$ as labeled by $y=\{1,...,c+1\}$. The probabilities predicted by $F$ provide weights $w$ to $V(G,D)$.}\label{fig:osda} 
\end{figure*}

\section{Methodology}
In this section, we first explain the fundamental settings of OSDA. Then, the details of the proposed method are given.
\subsection{Open Set Domain Adaptation}
In the standard settings of unsupervised OSDA~\cite{saito2018open}, labeled source data $\mathcal{S}=\{X^{S},Y^{S}\}$ and unlabeled target data $\mathcal{T}=\{X^{T}\}$ are given in the training phase. Here $X^{S}=\{x^s_i|i=1,2,\dots,n_s\}$, where $x^s_i$ denotes the $i$th source data and $n_s$ is the amount of the source data. The labels of the source data are $Y^{S}=\{y^s_i | y^s_i \in C^S\}$, where $C^S = \{1,2,\dots,c\}$ and $c$ is the number of the source category. Similarly, $X^{T}=\{x^t_j|j=1,2,\dots,n_t\}$, where $n_t$ is the amount of the target data. The labels of the target data $Y^T = \{y^t_j | y^t_j \in C^T\}$ are unavailable during the training process. Different from the standard settings of the traditional domain adaptation problem where the source data and the target data contain the same set of categories, the OSDA this paper refers to is that the target domain contains more categories than the source domain, i.e., $C^S \subset C^T$. Target data belonging to the unknown classes $C^{T\backslash S}$ is denoted as $X^T_u$ and target data belonging to the known categories is denoted as $X^T_k$. We need to distinguish $X^T_u$ from $X^T_k$ and $X^S$ while classifying $X^T_k$ simultaneously. 

\subsection{Against Adversarial Learning}
The network architecture of the proposed AAL method is shown in Figure~\ref{fig:osda}. The weighted adversarial domain adaptation method and a $c+1$-way classifier are used. The weighted adversarial learning maps target data and source data close in varying degrees according to the probability of the target sample belonging to the shared classes. By training against the adversarial learning, the $c+1$-way classifier aims to assign unknown target data to the $(c+1)$th class which is different from the source data. Following this ``against adversarial learning'' manner, distinguishing unknown target data and classifying the known target data can be achieved at the same time. In the following part of this section, the details of the proposed method are presented.

\subsubsection{Dynamic $c+1$-way Classifier}
Following the fact that we do not have any label information of the target data during the training process, all target data is labeled as ``unknown'' at first in the proposed method, i.e., $\tilde{Y}^T=\{\tilde{y}^t_j|\tilde{y}^t_j = c+1, j=1,2, ...,n_t\}$. Then the whole dataset we have can be denoted as $\{X,Y\}$, where $X={X^S}\cup {X^T}=\{x_i, i=1,2,...,n_s+n_t\}$ and $Y={Y^S}\cup {\tilde{Y}^T}=\{y_i, i=1,2,...,n_s+n_t\}$. A $c+1$-way \textit{\textbf{one-layer(i.e., without any hidden layer)}} classifier $F$ is added directly after the feature extractor $G$ of the adversarial learning architecture as shown in Figure~\ref{fig:osda}. At the beginning of the training process, $F$ aims to classify source data to the first $c$ classes according to their ground truth labels and classify all target data to the unknown class which is labeled as $c+1$. With the training of the network in Figure \ref{fig:osda}, if the ground truth label $y^t_j$ of target data $x^t_j$ belongs to one of the target-specific classes (i.e., $x^t_j\in X^T_u$), then though the adversarial training process maps target and source close, $F$ can still classify $x^t_j$ to the unknown class $c+1$. Otherwise, if $x^t_j$ belongs to one of the known classes from $1$ to $c$ (i.e., $x^t_j \in X^T_k$), it will be predicted to one of the known classes by $F$ under the influence of the adversarial learning though $F$ tries to classify it to category $c+1$ at the beginning. According to the labels of the target data predicted by $F$, the loss function of the dynamic classifier $F$ will keep changing along with the training process. The loss function of $F$ is given in eq. \ref{eq:loss_F}.
\begin{equation}
\begin{array}{lr}
- \mathcal{L}_{F}  = \frac{\alpha_i}{\sum\limits_{i}^{(n_s+n_t)}\alpha_i}\sum\limits_{i}^{(n_s+n_t)} log({p}_{i, y_i}) \\
\left\{
\begin{array}{lr}
\alpha_i = 1,  & \text{if } i \leq n_s \\
\alpha_i = 1,  & \text{if } \hat{y}_i = c+1 \text{ and } i > n_s \\
\alpha_i = 0,  & \text{if } \hat{y}_i < c+1 \text{ and } i > n_s.
\end{array}
\right.
\end{array}
\label{eq:loss_F}
\end{equation}
Here ${p}_{i, y_i}$ is the probability of data ${x}_i$ predicted by $F$ belonging to its ground truth label ${y}_i \in {Y}$. $\hat{y}_i$ is the predicted label of ${x}_i \in {X}$ output by $F$.  $\alpha_i$ is the weight of the $i$th data of $X$. As shown in eq.~\ref{eq:loss_F}, when the label of $x^t_j$ predicted by $F$ belongs to one of the known $c$ classes, the weight of $x^t_j$ in the loss function of $F$ will be set to $0$, i.e., $F$ will no longer aim to classify $x^t_j$ to the unknown class. The weight for source data will always be $1$. The outputs of $F$ are regarded as the final predicted results of the target data.

Classifier $F^*$ is trained on the labeled source data only. The loss function is
\begin{equation}
- \mathcal{L}_{F^*}  = \frac{1}{n_s}\sum_{i}^{n_s} log(p^*_{i, y_i}) \\
\label{eq:loss_F*}
\end{equation}
, where $p^*_{i,y_i}$ is the probability of $x^s_i$ belonging to its ground truth label $y^s_i$ predicted by $F^*$.
%

\subsubsection{Weighted Adversarial Domain Adaptation}
Let us use $G^{f}(x^s)$ and $G^{f}(x^t)$ to denote the deep features of source data and target data extracted by $G$ respectively.  ${F^*}^{p}(G^{f}(x^s))$ and ${F^*}^{p}(G^{f}(x^t))$ are the softmax probabilities distributed over each known class of the source features and target features predicted by $F^*$ respectively. For source data, deep features $f$ extracted by $G$ and the probability vectors $p^*$ extracted by $F^*$ follow the distribution $G_s$, i.e., $(f,p^*)\sim G_s(f,p^*)$. $f$ and $p^*$ of target data follow the distribution $G_t$, i.e., $(f,p^*)\sim G_t(f,p^*)$. $D(\cdot)$ denotes the probability of a data belonging to the source domain. $\otimes$ denotes the tensor product operation. Then, the objective function of the traditional conditional adversarial learning~\cite{long2018conditional} is
\begin{equation}
\begin{aligned}
\min\limits_G & \max\limits_{D} V(G,D) \\
= & \mathbb{E}_{x^{s}\sim p_{s}(x^{s})}[logD(G^{f}(x^{s}) \otimes {F^*}^{p}(G^{f}(x^{s})))] \\
 & +  \mathbb{E}_{x^{t}\sim p_{t}(x^t)}[log(1-D(G^{f}(x^{t}) \otimes {F^*}^{p}(G^{f}(x^t))))] \\
 = & \mathbb{E}_{(f,p^*)\sim G_{s}(f,p^*)}[logD(f \otimes p^*)] \\
 & + \mathbb{E}_{(f,p^*)\sim G_{t}(f,p^*)}[log(1-D(f \otimes p^*))].
\end{aligned}
\label{eq:cdan}
\end{equation}
Under the settings of OSDA, we need to distinguish the unknown target data from the known classes. Hence, the unknown target data should be mapped far away from the source classes compared with the target data belonging to the known classes. Thus, in the proposed method, the adversarial learning process is weighted by $W$ according to the probability of each target data belonging to the known classes, as shown in eq.~\ref{eq:wada}. 
\begin{equation}
\begin{aligned}
\min\limits_G\max\limits_{D} & V(G,D) \\ 
= & \mathbb{E}_{(f,p^*)\sim G_{s}(f,p^*)}[logD(f \otimes p^*)] \\
& + \mathbb{E}_{(f,p^*)\sim G_{t}(f,p^*)}[Wlog(1-D(f \otimes p^*))] 
\end{aligned}
\label{eq:wada}
\end{equation}
Using $p^d_{s_i}$ and $p^d_{t_j}$ to denote the probabilities of $x^s_i$ and $x^t_j$ belonging to the source domain output by $D$ respectively. $W_j$ denotes the probability of the $j$th target belonging to known classes. Then V(G,D) eq. \ref{eq:wada} can be expressed as 
\begin{equation}
\frac{1}{n_s}\sum^{n_s}_{i}log(p^d_{s_i}) + \sum^{n_t}_{j} \frac{W_j}{\sum_{j}^{n_t} W_j} log(1-p^d_{t_j})
\label{eq:loss_d}
\end{equation}
Denote the probability of target data $x^t_j$ belonging to the unknown categories predicted by $F$ as $w^u_j$, then $W_j=1 - w^u_j$. In this way, target data which has smaller probability belonging to the unknown class will have larger $W_j$, i.e., the adversarial learning process prefers to map target data more likely belonging to the known categories close to the source data. In order to avoid gradient vanishing problem in eq. \ref{eq:loss_d}, we set all $W_j$ to $1$ in the case where all $(1-w^u_j)$ are zeros in a batch of size $bs$ as shown in eq.~\ref{eq:ad_weight}.
\begin{equation}
\left\{
\begin{array}{lr}
	W_j = 1-w^u_j, & \text{if } \sum\limits_{j}^{bs}(1 - w^u_j) \neq 0 \\
	W_j = 1, & \text{if } \sum\limits_{j}^{bs}(1 - w^u_j) = 0.
\end{array}
\right.
\label{eq:ad_weight}
\end{equation}
The integrated objective function of the proposed method is 
\begin{equation}
\begin{aligned}
\min_{G,F,F^*} & \mathcal{L}_F + \mathcal{L}_{F^*} + V(G,D)\\
& \max_D V(G,D).
\end{aligned}
\end{equation}
The overall training process is given in Algorithm~\ref{algo_1}. Here $max\_iter$ is the total number of iterations. 
{\renewcommand\baselinestretch{0.8}\selectfont
	\renewcommand{\algorithmicrequire}{\textbf{Input:}}
	\renewcommand{\algorithmicensure}{\textbf{Procedure:}}
	\begin{algorithm}
		\caption{Against Adversarial Learning}\label{algo_1}
			
			\KwIn{Source dataset $\mathcal{S}$; Target data $\mathcal{T}$; Initial target labels $\tilde{Y}^T=\{y^t_j | y^t_j = c+1\}$;}
			\For {$i=1:max\_iter$} {
			 Extract $G^f(x^s)$ and $G^f(x^t)$; \\
			 Extract $F^p(G^f(x^s))$ and ${F^*}^p(G^f(x^s))$; \\
			 Input $G^f(x^s)\otimes {F^*}^p(G^f(x^s))$ and $G^f(x^t)\otimes {F^*}^p(G^f(x^t))$ to $D$; \\
		    Get target labels $\hat{y}^t$ predicted by $F$; \\
			Compute \textcolor{gray}{$W$ for the adversarial learning process} in eq. \ref{eq:wada} according to $F^p(G^f(x^s))$; \\
			Compute \textcolor{gray}{$\alpha$ for classifier $F$} in eq. \ref{eq:loss_F} according to $\hat{y}^t$; \\
			Update parameters \textcolor{blue}{$\theta_g$ of $G$} by \textcolor{blue}{eq. \ref{eq:loss_F}, eq. \ref{eq:loss_F*}, and eq. \ref{eq:wada}}; \\
		    Update parameters \textcolor{blue}{$\theta_d$ of $D$} by \textcolor{blue}{eq. \ref{eq:wada}}; \\
		}
			 Test $F$ on the target data; \\
			\KwOut{Average accuracy ``OS*'' of known classes and average accuracy ``OS'' of all classes including the unknown class predicted by $F$.}
	\end{algorithm}
	\par}

\subsection{Analysis}
We analyze the significance of using a $c+1$-way classifier to classify target data and a $c$-way classifier in the adversarial architecture separately instead of using a $c+1$-way classifier in the adversarial architecture directly. In the proposed method, the $c+1$-way classifier aims to map the unknown target data away from the known classes while the adversarial learning process maps target data close to the source data according to their probabilities of belonging to the known classes. If a $c+1$-way classifier is directly used in the adversarial architecture to classify the target data to the unknown class $c+1$, the input of $D$ will be $f\otimes \tilde{p}$, where $\tilde{p}$ is the probability vector with $c+1$ elements. Denote $\tilde{p}(i)$ as the $i$th element of $\tilde{p}$. Then the gradients back propagated to $G$ from $V(G,D)$ will be
\begin{equation}
\frac{\partial V(G,D)}{\partial f} = \sum_i^{c+1}\tilde{p}(i) \cdot \nabla \theta_d,
\label{eq:grad_2}
\end{equation}
where $\nabla \theta_d$ is the gradients of the discriminator propagated to $f\otimes \tilde{p}$. Eq.~\ref{eq:grad_2} implies that the target data will be mapped closer to the source data which belongs to category $\arg\max\tilde{p}$. As the target data is labeled as $c+1$ at first, we have $\arg\max\tilde{p} = c+1$. Thus, the adversarial learning cannot map target data and source data close since the labels of the source data are from $1$ to $c$. As a result, if $c+1$-way classifier is used in the adversarial directly, all target data will be classified to the unknown class. Thus, it is significant to build a $c+1$-way classifier to distinguish unknown data and known data and a $c$-way classifier for the adversarial architecture separately.

In addition, optimum of the weighted adversarial learning will be achieved when $G_{s}(f,p^*) = WG_{t}(f,p^*)$. The proof is given in Theorem 2 of Appendix A.

\begin{table*}[t]
	\centering
	\caption[Lof entry]{Accuracy on Office-31 for unsupervised open set domain adaptation. \cite{liu2019separate}}
	\renewcommand\arraystretch{1.0}
	\small
	\begin{tabular}{p{1.3cm}<{\centering}p{0.3cm}<{\centering}p{0.3cm}<{\centering}p{0.3cm}<{\centering}p{0.3cm}<{\centering}p{0.3cm}<{\centering}p{0.3cm}<{\centering}p{0.3cm}<{\centering}p{0.3cm}<{\centering}p{0.3cm}<{\centering}p{0.3cm}<{\centering}p{0.3cm}<{\centering}p{0.3cm}<{\centering}p{0.3cm}<{\centering}p{0.3cm}<{\centering}p{0.3cm}<{\centering}p{0.3cm}<{\centering}p{0.3cm}<{\centering}p{0.3cm}<{\centering}p{0.3cm}<{\centering}p{0.3cm}<{\centering}p{0.3cm}<{\centering}}
		\hlinew{1.5pt}		
		\multirow{2}{*}{Method} &
		\multicolumn{3}{c}{ A$\rightarrow$D } & \multicolumn{3}{c}{ A$\rightarrow$W } & \multicolumn{3}{c}{ D$\rightarrow$A } & \multicolumn{3}{c}{ D$\rightarrow$W } & \multicolumn{3}{c}{ W$\rightarrow$A } & \multicolumn{3}{c}{ W$\rightarrow$D } & \multicolumn{3}{c}{ Avg.}  \cr \cline{2-22}
		& OS & OS* & UN & OS & OS* & UN & OS & OS* & UN & OS & OS* & UN & OS & OS* & UN & OS & OS* & UN & OS & OS* & UN \\	
		\hline
		ResNet50 & 85.2 & 85.5 & 82.2 & 82.5 & 82.7 & 80.5 & 71.6 & 71.5 & 72.6 & 94.1 & 94.3 & 92.1 &75.5 & 75.2 & 78.5 & 96.6 & 97.0 & 92.6 & 84.2 & 84.4 & 83.1 \\
		RTN & 89.5 & 90.1 & 83.5 & 85.6 & 88.1 & 60.6 & 72.3 & 72.8 & 67.3 & 94.8  & 96.2 & 80.8 & 73.5 & 73.9 & 69.5 & 97.1 & 98.7 & 81.1 & 85.4 & 86.8 & 73.8 \\
		DANN & 86.5 & 87.7 & 74.5 & 85.3 & 87.7 & 61.3 & 75.7 & 76.2 & 70.7 & 97.5 & 98.3 & 89.5 & 74.9 & 75.6 & 67.9 & 99.5 & 100 & 94.5 & 86.6 & 87.6 & 76.4 \\
		OpenMax & 87.1 & 88.4 & 74.1 & 87.4 & 87.5 & 86.4 & 83.4 & 82.1 & 96.4 & 96.1 & 96.2 & 95.1 & 82.8 & 82.8 & 82.8 & 98.4 & 98.5 & 97.4 & 89.0 & 89.3 & 88.7 \\
		ATI-$\lambda$ & 84.3 & 86.6 & 61.3 & 87.4 & 88.9 & 72.4 & 78.0 & 79.6 & 62.0 & 93.6 & 95.3 & 76.6 & 80.4 & 81.4 & 70.4 & 96.5 & 98.7 & 74.5 & 86.7 & 88.4 & 69.5 \\
		OSBP & 88.6 & 89.2 & 82.6 & 86.5 & 87.6 & 75.5 & 88.9 & 90.6 & 71.9 & 97.0 & 96.5 & 96.5 & 85.8 & 84.9 & 94.8 & 97.9 & 98.7 & 89.9 & 90.8 & 91.3 & 85.2 \\
		STA & 93.7 & 96.1 & 69.7 & 89.5 & 92.1 & 63.5 & 89.1 & 93.5 & 45.1 & 97.5 & 96.5 & 99.5 & 87.9 & 87.4 & 92.9 & 99.5 & 99.6 & 98.5 & 92.9 & 94.1 & 78.2 \\
		\hline
		AAL & \textbf{97.5} & \textbf{100} & 72.5 & \textbf{92.1} & \textbf{94.3} & 70.1 & \textbf{93.4} & \textbf{94.9} & 78.4 & \textbf{98.4} & \textbf{100} & 82.4 & \textbf{95.1} & \textbf{96.7} & 79.1 & \textbf{100} & \textbf{100} & 100 & \textbf{96.1} & \textbf{97.7} & 80.4 \\
		\hlinew{1.5pt}
	\end{tabular}
	\label{tab:office_resnet50}
\end{table*}

\begin{table*}[t]
	\setlength{\belowcaptionskip}{0.5cm}
	\centering
	\caption[Lof entry]{Accuracy on Office-Home for unsupervised open set domain adaptation. \cite{liu2019separate}}
	\renewcommand\arraystretch{1.0}
	\small
	\begin{tabular}{p{1.3cm}<{\centering}p{0.3cm}<{\centering}p{0.3cm}<{\centering}p{0.3cm}<{\centering}p{0.3cm}<{\centering}p{0.3cm}<{\centering}p{0.3cm}<{\centering}p{0.3cm}<{\centering}p{0.3cm}<{\centering}p{0.3cm}<{\centering}p{0.3cm}<{\centering}p{0.3cm}<{\centering}p{0.3cm}<{\centering}p{0.3cm}<{\centering}p{0.3cm}<{\centering}p{0.3cm}<{\centering}p{0.3cm}<{\centering}p{0.3cm}<{\centering}p{0.3cm}<{\centering}p{0.3cm}<{\centering}p{0.3cm}<{\centering}p{0.3cm}<{\centering}}
		\hlinew{1.5pt}
		\multirow{2}{*}{Method} &
		\multicolumn{3}{c}{ Ar$\rightarrow$Cl } & \multicolumn{3}{c}{ Ar$\rightarrow$Pr } & \multicolumn{3}{c}{ Ar$\rightarrow$Rw } & \multicolumn{3}{c}{ Cl$\rightarrow$Ar } & \multicolumn{3}{c}{ Cl$\rightarrow$Pr } & \multicolumn{3}{c}{ Cl$\rightarrow$Rw } & \multicolumn{3}{c}{ -- } \cr \cline{2-22}
		& OS & OS* & UN & OS & OS* & UN & OS & OS* & UN & OS & OS* & UN & OS & OS* & UN & OS & OS* & UN & -- & -- & -- \\	
		\hline
		ResNet50 & 53.4 & -- & -- & 69.3 & -- & -- & 78.7 & -- & -- & 61.4 & -- & -- & 61.8 & -- & -- & 71.0 & -- & -- & -- & -- & -- \\
		DANN & 54.6 & -- & -- & 69.5 & -- & -- & 80.2 & -- & -- & 61.9 & -- & -- & 63.5 & -- & -- & 71.7 & -- & -- & -- & -- & -- \\
		OpenMax & 56.5 & -- & -- & 69.1 & -- & -- & 80.3 & -- & -- & 64.1 & -- & -- & 64.8 & -- & -- & 73.0 & -- & -- & -- & -- & -- \\
		ATI-$\lambda$ & 55.2 & -- & -- & 69.1 & -- & -- & 79.2 & -- & -- & 61.7 & -- & -- & 63.5 & -- & -- & 72.9 & -- & -- & -- & -- & -- \\
		OSBP & 56.7 & -- & -- & 67.5 & -- & -- & 80.6 & -- & -- & 62.5 & -- & -- & 65.5 & -- & -- & 74.7 & -- & -- & -- & -- & -- \\
		STA & 58.1 & -- & -- & 71.6 & -- & -- & 85.0 & -- & -- & 63.4 & -- & -- & 69.3 & -- & -- & 75.8 & -- & -- & -- & -- & -- \\
		\hline
		AAL & \textbf{69.4} & \textbf{69.3} & 71.9 & \textbf{80.2} & \textbf{80.7} & 67.7 & \textbf{85.7} & \textbf{86.1} & 75.7 & \textbf{74.5} & \textbf{75.0} & 62.0 & \textbf{76.2} & \textbf{76.4} & 71.2 & \textbf{78.4} & \textbf{78.9} & 65.9 & -- & -- & -- \\
		\hline
		\hline
		\multirow{2}{*}{Method} &
		\multicolumn{3}{c}{ Pr$\rightarrow$Ar } & \multicolumn{3}{c}{ Pr$\rightarrow$Cl } & \multicolumn{3}{c}{ Pr$\rightarrow$Rw } & \multicolumn{3}{c}{ Rw$\rightarrow$Ar } & \multicolumn{3}{c}{ Rw$\rightarrow$Cl } & \multicolumn{3}{c}{ Rw$\rightarrow$Pr } & \multicolumn{3}{c}{ Avg. } \cr \cline{2-22}
		& OS & OS* & UN & OS & OS* & UN & OS & OS* & UN & OS & OS* & UN & OS & OS* & UN & OS & OS* & UN & OS & OS* & UN \\	
		\hline
		ResNet50 & 64.0 & -- & -- & 52.7 & -- & -- & 74.9 & -- & -- & 70.0 & -- & -- & 51.9 & -- & -- & 74.1 & -- & -- & 65.3 & -- & -- \\
		DANN & 63.3 & -- & -- & 49.7 & -- & -- & 74.2 & -- & -- & 71.3 & -- & -- & 51.9 & -- & -- & 72.9 & -- & -- & 65.4 & -- & -- \\
		OpenMax & 64.0 & -- & -- & 52.9 & -- & -- & 76.9 & -- & -- & 71.2 & -- & -- & 53.7 & -- & -- & 74.5 & -- & -- & 66.7 & -- & -- \\
		ATI-$\lambda$ & 64.5 & -- & -- & 52.6 & -- & -- & 75.8 & -- & -- & 70.7 & -- & -- & 53.5 & -- & -- & 74.1 & -- & -- & 66.1 & -- & -- \\
		OSBP & 64.8 & -- & -- & 51.5 & -- & -- & 71.5 & -- & -- & 69.3 & -- & -- & 49.2 & -- & -- & 74.0 & -- & -- & 65.7 & -- & -- \\
		STA & 65.2 & -- & -- & 53.1 & -- & -- & 80.8 & -- & -- & 74.9 & -- & -- & 54.4 & -- & -- & 81.9 & -- & -- & 69.5 & -- & -- \\
		\hline
		AAL & \textbf{71.1} & \textbf{71.3} & 66.1 & \textbf{64.3} & \textbf{64.1} & 69.3 & \textbf{83.1} & \textbf{83.5} & 73.1 & \textbf{79.6} & \textbf{80.0} & 69.6 & \textbf{64.2} & \textbf{64.0} & 69.2 & \textbf{84.7} & \textbf{84.9} & 79.7 & \textbf{76.0} & \textbf{76.2} & 71.0 \\
		\hlinew{1.5pt}
	\end{tabular}
	\label{tab:officehome_resnet50}
\end{table*}

\section{Experiments}
In this section, the experimental settings are firstly given. Then the proposed method is evaluated by comparing with several state-of-the-art methods. Some qualitative analysis is also given. The codes will be available in Github.

\subsection{Experimental Setup}
\textbf{1) Datasets} The \textbf{Office-31} dataset\cprotect\footnote{https://people.eecs.berkeley.edu/\Verb!~!jhoffman/domainadapt/\\ \Verb!#!datasets$\_$code}~\cite{saenko2010adapting} contains images originated from three domains: Amazon (A), Webcam (W), and DSLR (D). These three domains consist of the same 31 categories. Following the standard settings of the OSDA, the first 10 categories are set as the shared classes of the source and the target domain and class $21\sim 31$ are the target-specific unknown categories. By using Office-31, we can evaluate the proposed method on 6 transfer tasks: $W\rightarrow D$, $W\rightarrow A$, $D\rightarrow W$, $D\rightarrow A$, $A\rightarrow W$, and $A\rightarrow D$.

The \textbf{Office-Home}\footnote{http://hemanthdv.org/OfficeHome-Dataset/}~\cite{venkateswara2017deep} dataset consists of 4 domains: Artistic (Ar), Clip Art (Cl), Product (Pr) and Real-World (Rw). There are 65 categories in each domain and more than 15,000 images in total. The first 25 categories are set as the shared classes while the remaining classes are the target-specific unknown categories. Compared with Office-31, this is a more challenging dataset for domain adaptation evaluation because each domain contains more categories and images in each category have significant domain shifts. For this dataset, 12 transfer tasks can be generated for evaluation using all its 4 domains: $Cl\rightarrow Pr$, $Cl\rightarrow Rw$, $Cl\rightarrow Ar$, $Pr\rightarrow Cl$, $Pr\rightarrow Rw$, $Pr\rightarrow Ar$, $Rw\rightarrow Cl$, $Rw\rightarrow Pr$, $Rw\rightarrow Ar$, $Ar\rightarrow Cl$, $Ar\rightarrow Pr$, and $Ar\rightarrow Rw$.

\textbf{2) Network Architecture}
The feature extractor $G$ in our experiments was built based on the architecture of ResNet50~\cite{he2016deep}. To fairly compare with other adversarial methods, a bottleneck layer with size 256 is added before the last full-connected layer. The tensor products of the $softmax$ probabilities $p^*$ and the outputs of the $bottleneck$ layer are utilized as the inputs of $D$. The discriminator $D$ used in our experiments consists of three fully connected layers. The size of the first two layers are 1024 followed by ReLU activation layer and dropout layer while the dimension of final outputs is 1. The $c$-way classifier $F^*$ and the $c+1$-way classifier $F$ are one-layer classifiers, which means there are no hidden layers. The size of the weights of $F$ and $F^*$ are $256\times (c+1)$ and $256\times c$ respectively. 

\textbf{3) Training Process} 
Models in our experiments were trained using the framework \textbf{Pytorch} using GPU \textbf{Tesla V100 32G} on Linux system. Following the standard fine-tuning procedure, the learning rates of the first few layers were set to a small number to slightly tune the parameters initialized from the pre-trained model. The learning rate for the other layers like $bottleneck$ and the last fully connected layer can be set larger, typically 10 times that of the lower layers. We used the stochastic gradient descent (SGD) update strategy with a momentum of 0.9. The base learning rates of classifier $F$, $F^*$ and the adversarial learning architecture were both initialized to 0.0001 for all the tasks of Office-31. The base learning rate of $F$ for all tasks of Office-Home dataset were set to 0.001 while that of $F^*$ and the adversarial architecture was 0.0001. During the training process, the learning rate was changed by the following strategy: $base\_lr\times (1+\gamma \times iter)^{-power}$, where $power$ was set to 0.75 throughout all experiments and $iter$ is the current number of iterations. $\gamma=0.0001$ was used for $F$ on Office-Home while $\gamma=0.001$ is used for all the other cases. All components are trained at the same time since the proposed architecture is end-to-end. The max training iterations are set to $20000$. Following previous works \cite{panareda2017open} \cite{saito2018open} \cite{liu2019separate}, each task was run for $3$ times and the average accuracy is the final result.
\begin{figure*}[t]
	\centering
	\subfigure[$A\rightarrow D$]{
		\includegraphics[width=0.32\linewidth]{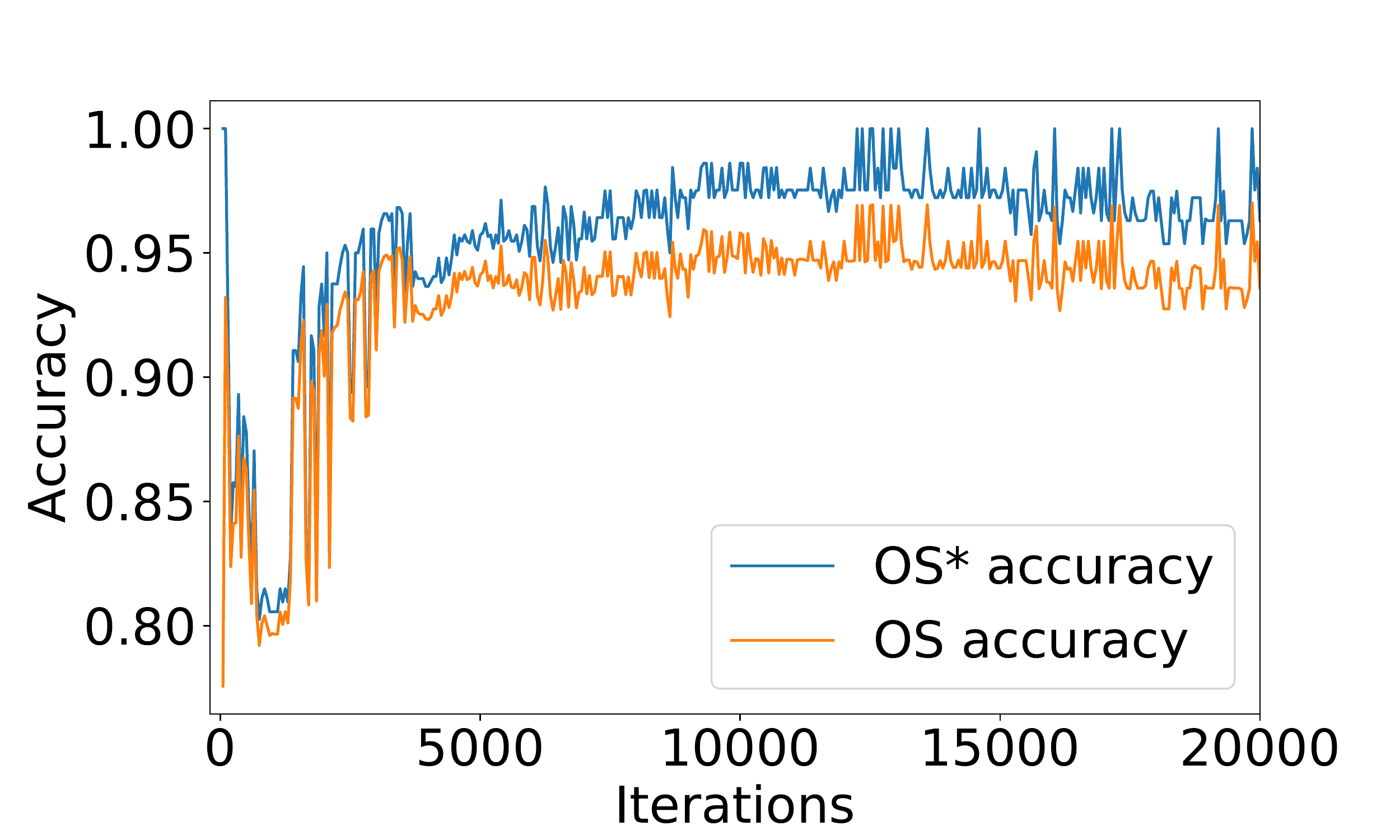}
	} 
	\subfigure[$Ar\rightarrow Cl$]{
		\includegraphics[width=0.32\linewidth]{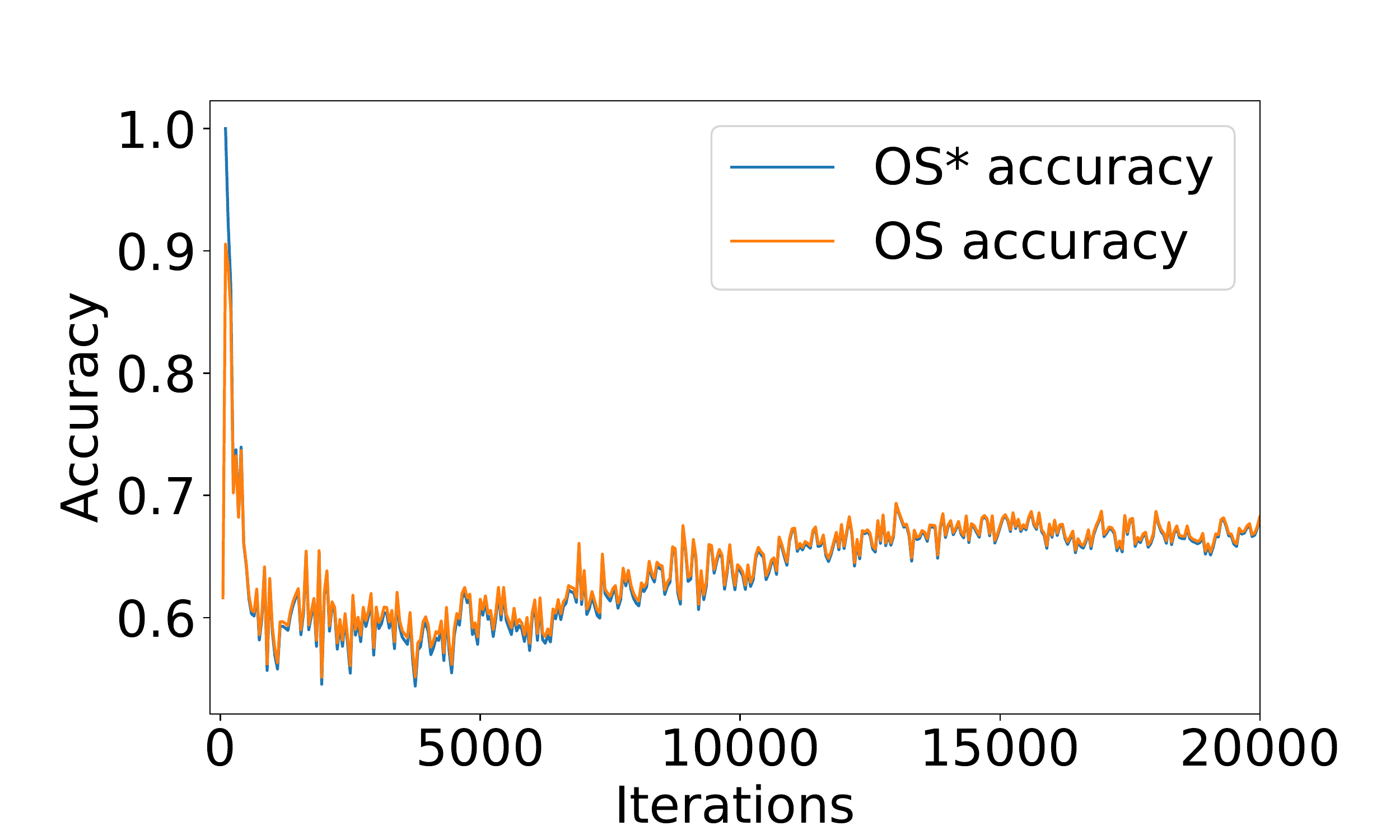}
	} 
	\subfigure[$Rw\rightarrow Pr$]{
		\includegraphics[width=0.32\linewidth]{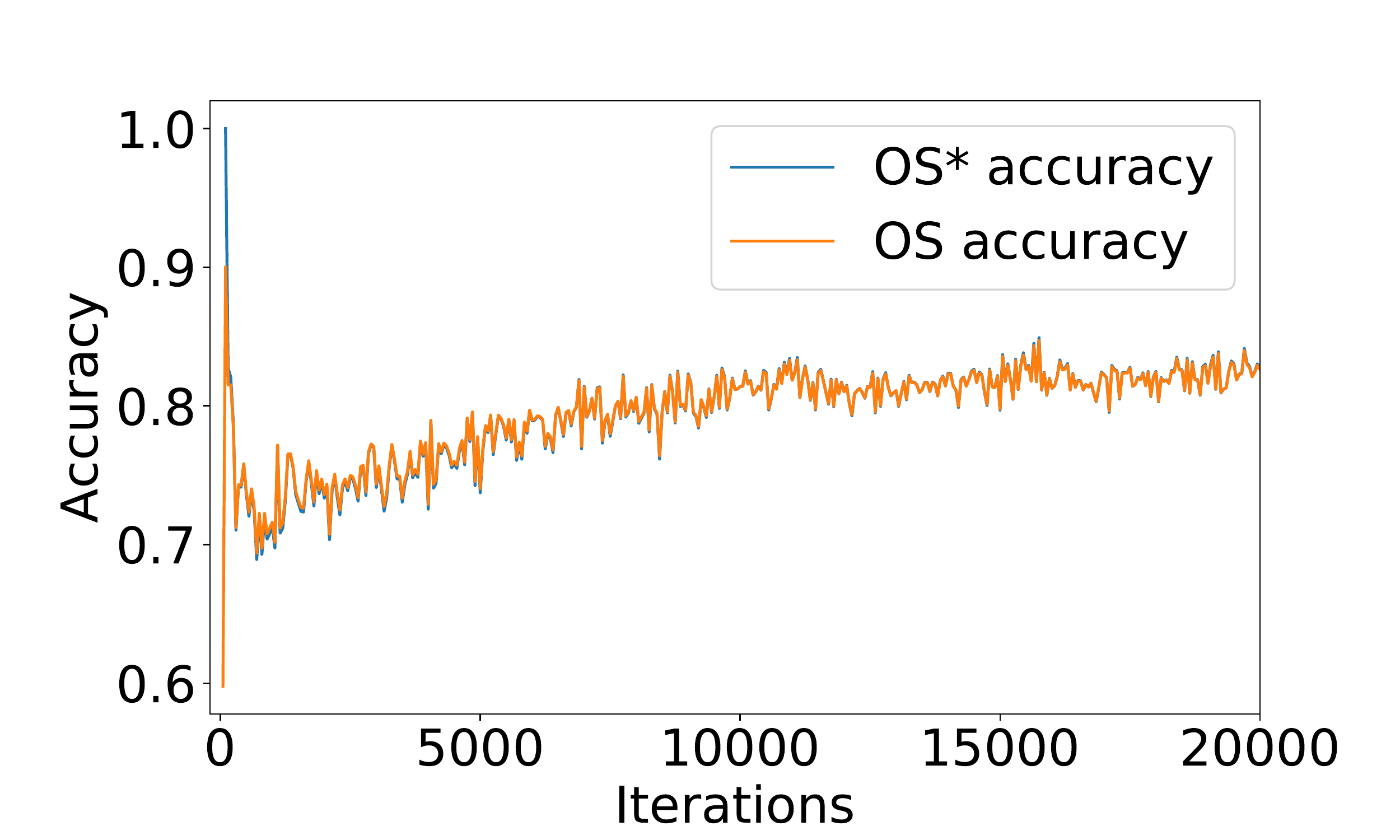}
	}
	
	\caption{(a), (b), and (c) are the accuracy curves of task A$\rightarrow$D, Ar$\rightarrow$Cl, and Rw$\rightarrow$Pr respectively. The orange curve denotes the accuracies evaluated by metric OS of each iteration. The blue curve denotes the accuracies evaluated by OS*.}\label{fig:acc}
\end{figure*}

\begin{figure}[t]
	\centering
	\subfigure[$A\rightarrow W$]{
		\includegraphics[width=0.47\linewidth, height=0.45\linewidth]{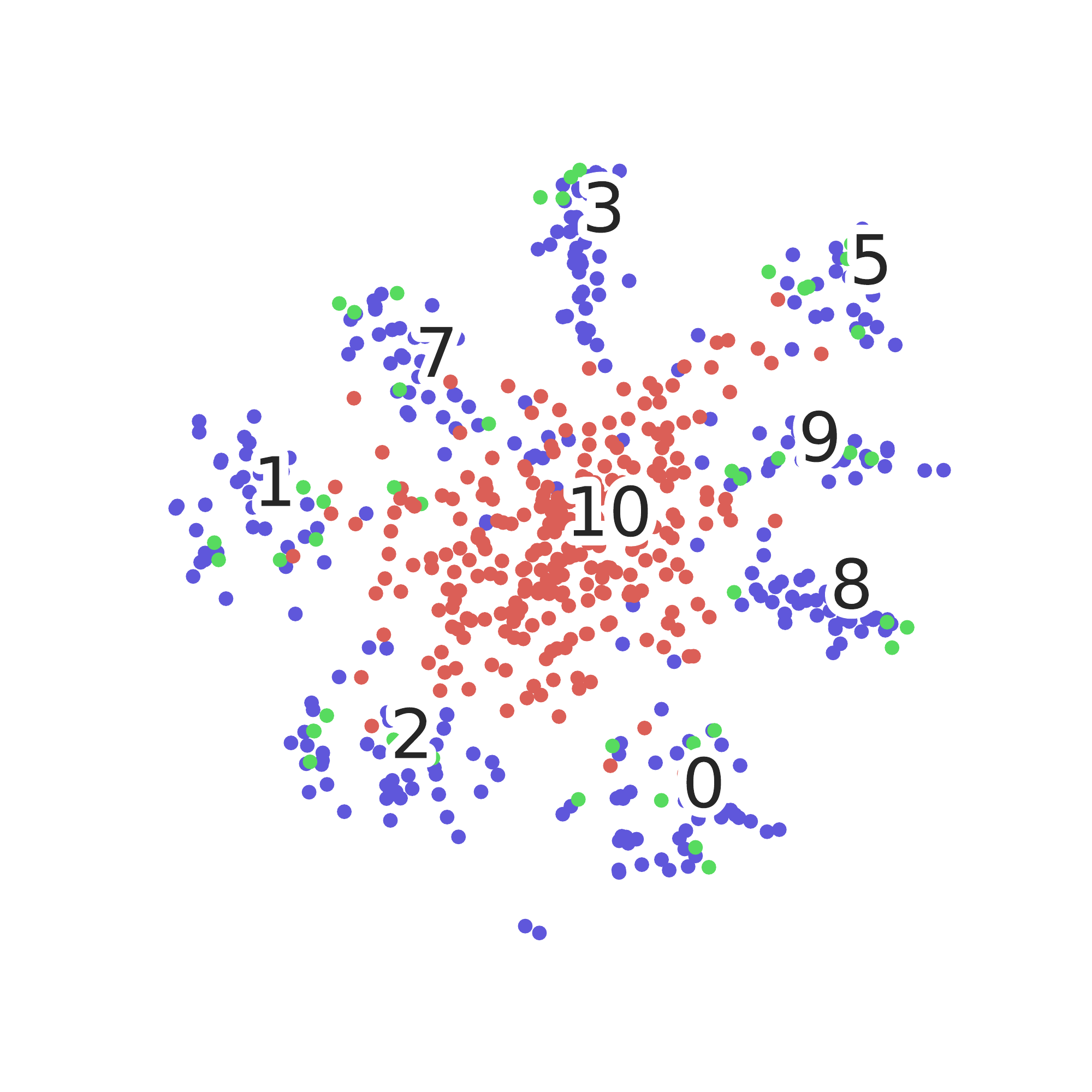}
	} 
	\subfigure[$A\rightarrow D$]{
		\includegraphics[width=0.47\linewidth, height=0.45\linewidth]{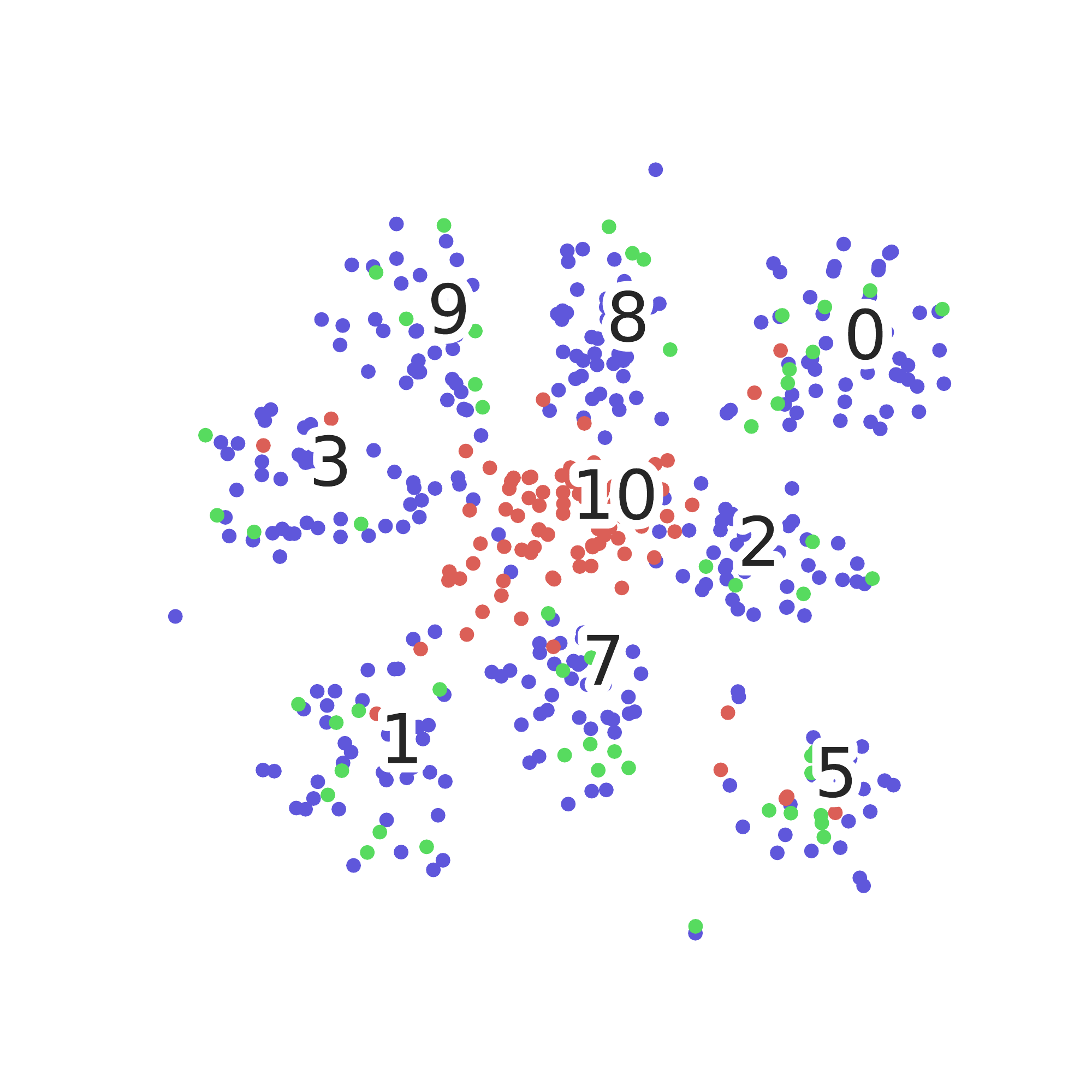}
	} 
	\caption{(a) and (b) are t-SNE embeddings of the deep features of task A$\rightarrow$W and A$\rightarrow$D respectively. Known target data is in green, source data is denoted in blue, and the unknown data is in red. }\label{fig:tsne}
\end{figure}

\subsection{Results}
To fairly compare with previous works, two metrics are used to evaluate the proposed method: OS and OS*. OS is the average accuracy for all $c+1$ classes and OS* is the average accuracy only on the $c$ known classes. Following the strategy of previous works \cite{liu2019separate}, in the case where there are no target data predicted to class $k$, then class $k$ will not be taken account when computing OS and OS*. Several previous domain adaptation methods are compared in our experiments. Baseline methods ResNet50 \cite{he2016deep}, RTN \cite{long2016unsupervised} and DANN \cite{ganin2016domain} are developed for closed set domain adaptation. By combining them with open SVM \cite{jain2014multi}, they can be used for OSDA. OpenMax \cite{bendale2016towards}, ATI-$\lambda$ \cite{panareda2017open}, OSBP \cite{saito2018open}, and STA \cite{liu2019separate} are currently state-of-the-art methods designed for OSDA.

We denote the proposed method as ``AAL''. The results obtained for the Office-31 are summarized in Table~\ref{tab:office_resnet50}. The best results are highlighted in bold. From Table~\ref{tab:office_resnet50}, we can observe that the proposed method can achieve significant improvement compared with the previous state-of-the-art methods in all the six tasks evaluated by both metrics OS and OS*. In some tasks like $A\rightarrow D$ and $D\rightarrow W$, the proposed ``AAL'' achieves an accuracy of $100\%$ for OS*, which indicates that it classified all target data predicted to the known classes correctly and no unknown target data were classified to the known classes. For task $W\rightarrow D$, the proposed ``AAL'' classified all target data to their corresponding classes including the unknown target data. 

The experimental results of the Office-Home dataset are recorded in Table~\ref{tab:officehome_resnet50}. As the results evaluated using OS* have not been reported in previous works, we only list results of OS for the previous state-of-the-art methods. For this more challenging domain adaptation dataset, the proposed method is still very distinctive in the domain adaptation performance, which shows the effectiveness of the proposed method. In some tasks, the proposed method can achieve an improvement of about  $10\%$, e.g., $Ar\rightarrow Cl$, $Cl\rightarrow Ar$, $Pr\rightarrow Cl$, $Rw\rightarrow Cl$. In addition, as shown in our experiments, methods specifically designed for OSDA (e.g., OpenMax, OSBP, STA) can achieve better performance than methods for closed set domain adaptation (e.g., RTN, DANN). This indicates that traditional methods are not suitable for solving OSDA problem since they tend to map all source features and target features close.

Besides standard measurements OS and OS* used in previous works, we also report the accuracy of the unknown class which is denoted as ``UN''  in Table \ref{tab:office_resnet50} and Table \ref{tab:officehome_resnet50}. When only a few target data which are most close to the source data are treated as known, OS and OS* will be high but UN will be low. On the contrary, if too many unknown target data are predicted to known classes, then UN will be high but OS and OS* will be low. To fairly compare with previous state-of-the-art methods, here we only use ``UN'' to reflect whether a model tries to classify all target data as unknown and thus achieves good performance in OS and OS*. We can observe that though OpenMax and OSBP achieve higher ``UN'' accuracies than ``AAL'', their OS and OS* accuracies are obviously lower than ``AAL'', which means too many unknown target data are predicted as known. As far as we know, we are the first to report accuracy ``UN''. From Table \ref{tab:office_resnet50} and Table \ref{tab:officehome_resnet50}, we can observe that our proposed method can make significant improvements to ``OS'' and ``OS*'' accuracies while maintaining relatively high accuracy for the unknown class at the same time, which further indicates the effectiveness of the proposed ``AAL'' method. In some tasks (e.g., Ar$\rightarrow$Cl etc.) of Table \ref{tab:officehome_resnet50}, UN accuracies are even higher than that of the known classes when their OS* accuracies has already been much higher that previous state-of-the-art methods.

\subsection{Further Analysis}
\subsubsection{Convergence}
OS and OS* accuracies of each iteration for task $A\rightarrow D$, $Ar\rightarrow Cl$, and $Rw\rightarrow Pr$ are plotted in Figure~\ref{fig:acc}. In the first few iterations, since target features and source features are far away, a lot of known target data are classified as unknown at this time. Thus the accuracy of the unknown class is small, which results in low OS accuracies. At the same time, only target data that are most close to the source data can be classified to known classes in the first few iterations, thus these target data are most likely to be correctly classified. Therefore, the accuracies given by OS* are extremely high at first ($100\%$ sometimes). Along with the increase of the training iterations, the deep network tend to be stable. More target data are classified correctly as a consequent. Thus, both OS and OS* output higher accuracies (more detailed analysis could be found in Appendix B). In Figure \ref{fig:acc} (a), it is normal that the accuracy of task A$\rightarrow$D fluctuates more widely. This is because that domain D contains a small amount of data, even a change of one predicted label can have a big impact.

\begin{table}[t]
	\setlength{\belowcaptionskip}{0.5cm}
	\centering
	\caption[LoF entry]{Ablation study of the proposed ``AAL'' without $W$ or $\alpha$.}
	\renewcommand\arraystretch{1.0}
	\small
	\begin{tabular}{p{0.8cm}<{\centering}p{0.6cm}<{\centering}p{0.6cm}<{\centering}p{0.6cm}<{\centering}p{0.6cm}<{\centering}p{0.6cm}<{\centering}p{0.6cm}<{\centering}p{0.6cm}<{\centering}}
		\hline
		\multirow{2}{*}{$W$} &
		\multirow{2}{*}{$\alpha$} & \multicolumn{3}{c}{ A$\rightarrow$W } & \multicolumn{3}{c}{ D$\rightarrow$A } \cr \cline{3-8} 
		& & OS & OS* & UN & OS & OS* & UN \\	
		\hline
		\ding{55} & \ding{55} & 94.8 & 100 & 42.8 & 71.2 & 72.6 & 57.2 \\
		\ding{55} & \ding{51} & 83.0 & 81.7 & 96.0 & 78.3 & 77.6 & 85.3 \\			
		\ding{51} & \ding{55} & 89.4 & 92.1 & 62.4 & 73.6 & 75.2 & 57.6 \\
		\ding{51} & \ding{51} & 92.1 & 94.3 & 70.1 & 93.4 & 94.9 & 78.4 \\
		\hline		
		\multirow{2}{*}{$W$} &
		\multirow{2}{*}{$\alpha$} & \multicolumn{3}{c}{ Ar$\rightarrow$Rw } & \multicolumn{3}{c}{ Cl$\rightarrow$Ar }\cr \cline{3-8} 
		& & OS & OS* & UN & OS & OS* & UN \\	
		\hline
		\ding{55} & \ding{55} & 82.6 & 83.4 & 62.6 & 74.6 & 75.4 & 54.6 \\
		\ding{55} & \ding{51} & 79.2 & 79.1 & 81.7 & 71.6 & 71.7 & 69.1 \\			
		\ding{51} & \ding{55} & 90.4 & 91.8 & 55.4 & 79.2 & 80.0 & 59.2 \\
		\ding{51} & \ding{51} & 85.7 & 86.1 & 75.7 & 74.5 & 75.0 & 62.0 \\
		\hline		
	\end{tabular}
	\label{tab:ablation}
\end{table}
\subsubsection{t-SNE}
To visually show the distributions of deep features extracted by the proposed method, we plot the t-SNE mappings of the extracted deep features in Figure~\ref{fig:tsne} for task $A\rightarrow W$ and $A\rightarrow D$. We select to plot features of $8$ known classes and the unknown class, the labels of which are annotated in Figure~\ref{fig:tsne}. The unknown class are displayed in red (class $10$ in Figure~\ref{fig:tsne} (a) and Figure~\ref{fig:tsne} (b)). From Figure~\ref{fig:tsne}, we can observe that the proposed AAL can map features of the target data belonging to unknown class differently from the known data. Two more examples selected from Office-Home dataset are displayed in Appendix C.

\subsubsection{Ablation Study}
The influence of different parts in the proposed method is evaluated in this section. The ablation study results are displayed in Table \ref{tab:ablation}. $W$ is the weight for adversarial learning and $\alpha$ denotes the weight for classifier $F$. A ``\ding{51}'' means the corresponding weight is included while a ``\ding{55}'' means that it is not applied. Here 4 transfer tasks are selected as examples, 2 from Office dataset (i.e., A$\rightarrow$W and D$\rightarrow$A) and 2 from Office-Home dataset (i.e., Ar$\rightarrow$Rw and Cl$\rightarrow$Ar). \textbf{Case 1): No $W$ and no $\alpha$.} From Table \ref{tab:ablation} we can observe that task A$\rightarrow$W seems to achieve better accuracies without $W$ and $\alpha$. However, a huge gap between OS accuracy and OS* accuracy indicates an extremely low accuracy for the unknown class ``UN'' ($42.8\%$), which means that a lot of known target data are classified as unknown, i.e., the network failed to recognize known target data. The situation for task Cl$\rightarrow$Ar is similar, which slightly improves the accuracies by sacrificing the ability of recognizing known target data. The other tasks have low accuracies for all OS, OS*, and UN. \textbf{Case 2): No $W$ but have $\alpha$}. Without $W$, the adversarial learning process will map all source and target close. At the same time, using $\alpha$ can help to map known target data and source data close, which can further classify more target data to known classes. Consequently, two many target data are classified to known classes which results in performance decline of OS and OS*. \textbf{Case 3): Have $W$ but no $\alpha$}. Classifier $F$ will always tries to classified all target data to the unknown class $c+1$. Thus, target features are far from those of the source, a lot of target features are misclassified to unknown class. From all these 3 cases, we can imply that a combination of $W$ and $\alpha$ can make the model perform well and stable for OS, OS*, and UN. Both $W$ and $\alpha$ are significant for the proposed method. The convergence curve when there are no $W$ and $\alpha$ is given in Appendix D.

\section{Conclusion}
In this paper, we present an open set domain adaptation approach called ``against adversarial learning'' (AAL) which can naturally distinguish unknown target data and known data without any additional hyper parameters. Following the fact that labels of all target data are unavailable during the training process, we consider all target data as unknown at the beginning. Then, by using a dynamic classifier training against a weighted adversarial learning architecture, the proposed method can effectively distinguish the unknown target data from the known classes and classify the known target data at the same time. Analyses from a variety of aspects are given. Experimental results show that the proposed method can achieve significant improvements compared with the previous state-of-the-art methods.

\cleardoublepage
\bibliography{osda}
\clearpage
\begin{appendices}
	\section{Appendix A. Proof of optimum}
	In this section, we give first prove that the optimum of the traditional conditional adversarial domain adaptation will be reached when conditional distribution of source data and target data are the same. Then, we will further give the proof for the optimum of the weighted adversarial learning.
	
	\begin{theorem}
		Given the objective function $V(G,D)$ in eq.~\ref{eq:cdan} and following the proof of ``Proposition 1'' in~\cite{goodfellow2014generative}, \textit{for any fixed $G$, the optimal discriminator $D$ in eq.~\ref{eq:cdan} is }
		\begin{equation}
			D^{*}_{G}(f \otimes p^*) = \frac{G_{s}(f, p^*)}{G_{s}(f,p^*) + G_{t}(f,p^*)}
			\label{eq:op_D}
		\end{equation}
	\end{theorem}
	
	\begin{proof}	
		For eq. \ref{eq:cdan}, given any fixed generator $G$, the discriminator $D$ is trained to maximize the value function $V(G,D)$:
		\begin{equation}
			\begin{aligned}
				\min\limits_G \max\limits_D & V(G,D) \\
				= & \int_{x^{s}} p_{s}(x^{s})[logD(G^{f}(x^{s}) \otimes {F^*}^{p}(x^{s}))]d_{x^s} \\
				+ & \int_{x^{t}} p_{t}(x^t)[log(1-D(G^{f}(x^{t}) \otimes {F^*}^{p}(x^t)))]d_{x^t} \\
				= & \int_{f}\int_{p^*} G_{s}(f,p^*) logD(f \otimes p^*) \\
				+ & G_{t}(f,p^*) log(1-D(f \otimes p^*)) d_f d_{p^*} \\
			\end{aligned}
			\label{eq:optimal_D}
		\end{equation}
		where $G_{s}(f,p^*) = (G^{f}(x^s),{F^*}^{p}(x^s))_{x^s \sim p_{s}(x^s)}$ and $G_{t}(f,p^*) = (G^{f}(x^t),{F^*}^{p}(x^t))_{x^t \sim p_{t}(x^t)}$. Eq.~\ref{eq:optimal_D} has the same form as function $y \rightarrow a\, log(y) + b\,log(1 - y), (a,b) \in \mathbb{R}^2 \setminus \{0,0\}$, which achieves its maximum at $\frac{a}{a + b} \in [0,1]$. So similarly, given $G$ fixed, the optimal $D$ that makes $V(G,D)$ achieve its maximum can be obtained as in eq.~\ref{eq:op_D}. 
	\end{proof}
	Then, by substituting eq. \ref{eq:op_D} into eq.~\ref{eq:cdan}, the training criterion for $G$ is to minimize
	
	\begin{equation}
		\begin{aligned}
			& V(G,D^{*}_{G}) \\
			& = \mathbb{E}_{x^{s}\sim p_{s}(x^{s})}[logD^{*}_{G}(G^{f}(x^{s}) \otimes {F^*}^{p}(x^{s}))] \\
			& + \mathbb{E}_{x^{t}\sim p_{t}(x^{t})}[log(1-D^{*}_{G}(G^{f}(x^{t}) \otimes {F^*}^{p}(x^t)))] \\
			& = \mathbb{E}_{(f,p^*) \sim G_{s}(f,p^*)}[logD^{*}_{G}(f \otimes p^*)] \\
			& + \mathbb{E}_{(f,p^*) \sim G_{t}(f,p^*)}[log(1-D^{*}_{G}(f \otimes p^*))] \\
			& = \mathbb{E}_{(f,p^*) \sim G_{s}(f,p^*)}[log\frac{G_{s}(f,p^*)}{G_{s}(f,p^*) + G_{t}(f,p^*)}] \\
			& + \mathbb{E}_{(f,p^*) \sim G_{t}(f,p^*)}[log\frac{G_{t}(f,p^*)}{G_{s}(f,p^*) + G_{t}(f,p^*)}] \\
		\end{aligned}
		\label{eq:min_G}
	\end{equation}
	According to~\cite{goodfellow2014generative}, it is straightforward to induce that eq.~\ref{eq:min_G} can be reformulated to
	\begin{equation}
		V(G,D^{*}_{G}) = -log(4) + 2\cdot JSD(G_{s}(f,p^*)\parallel G_{t}(f,p^*))
	\end{equation}
	We can see that when $G_{s}(f,p^*) = G_{t}(f,p^*)$, the global minimum can be achieved as the Jensen-Shannon divergence (JSD) between two distributions is always non-negative and equals to zero iff they are exactly the same. To sum up, in this adversarial architecture, the deep neural network $G$ tends to generate equally distributed probability-feature joint outputs for the target and source data:
	\begin{equation}
		\begin{aligned}
			p(G^{f}(x^s))p({F^*}^{p}(x^s)|G^{f}(x^s)) = \\
			p(G^{f}(x^t))p({F^*}^{p}(x^t)|G^{f}(x^t)).
		\end{aligned}
		\label{eq:distribution}
	\end{equation} 
	
	Then, the proof of the optimum for weighted adversarial domain adaptation eq. \ref{eq:wada} is given.
	\begin{theorem}
		Given the objective function $V(G,D)$ in eq.~\ref{eq:wada} and following the proof of ``Proposition 1'' in~\cite{goodfellow2014generative}, \textit{for any fixed $G$, the optimal discriminator $D$ in eq.~\ref{eq:wada} is }
		\begin{equation}
			D^{*}_{G}(f \otimes p^*) = \frac{G_{s}(f, p^*)}{G_{s}(f,p^*) + WG_{t}(f,p^*)}
			\label{eq:op_D_weighted}
		\end{equation}
	\end{theorem}
	
	\begin{proof}	
		For eq. \ref{eq:wada}, given any fixed generator $G$, the discriminator $D$ is trained to maximize the value function $V(G,D)$:
		\begin{equation}
			\begin{aligned}
				\min\limits_G \max\limits_D & V(G,D) \\
				= & \int_{x^{s}} p_{s}(x^{s})[logD(G^{f}(x^{s}) \otimes {F^*}^{p}(x^{s}))]d_{x^s} \\
				+ & \int_{x^{t}} p_{t}(x^t)[Wlog(1-D(G^{f}(x^{t}) \otimes {F^*}^{p}(x^t)))]d_{x^t} \\
				= & \int_{f}\int_{p^*} G_{s}(f,p^*) logD(f \otimes p^*) \\
				+ & WG_{t}(f,p^*) log(1-D(f \otimes p^*)) d_f d_{p^*} \\
			\end{aligned}
			\label{eq:optimal_D_weighted}
		\end{equation}
		where $G_{s}(f,p^*) = (G^{f}(x^s),{F^*}^{p}(x^s))_{x^s \sim p_{s}(x^s)}$ and $G_{t}(f,p^*) = (G^{f}(x^t),{F^*}^{p}(x^t))_{x^t \sim p_{t}(x^t)}$. Eq.~\ref{eq:optimal_D_weighted} has the same form as function $y \rightarrow a\, log(y) + b\,log(1 - y), (a,b) \in \mathbb{R}^2 \setminus \{0,0\}$, which achieves its maximum at $\frac{a}{a + b} \in [0,1]$. So similarly, given $G$ fixed, the optimal $D$ that makes $V(G,D)$ in eq. \ref{eq:wada} achieve its maximum can be obtained as in eq.~\ref{eq:op_D_weighted}. 
	\end{proof}
	Then, by substituting the optimal discriminator in eq. \ref{eq:op_D_weighted} into eq.~\ref{eq:wada}, the training criterion for $G$ is to minimize
	\begin{equation}
		\begin{aligned}
			& V(G,D^{*}_{G}) \\
			& = \mathbb{E}_{x^{s}\sim p_{s}(x^{s})}[logD^{*}_{G}(G^{f}(x^{s}) \otimes {F^*}^{p}(x^{s}))] \\
			& + \mathbb{E}_{x^{t}\sim p_{t}(x^{t})}[Wlog(1-D^{*}_{G}(G^{f}(x^{t}) \otimes {F^*}^{p}(x^t)))] \\
			& = \mathbb{E}_{(f,p^*) \sim G_{s}}[logD^{*}_{G}(f \otimes p^*)] \\
			& + \mathbb{E}_{(f,p^*) \sim G_{t}}[Wlog(1-D^{*}_{G}(f \otimes p^*))] \\
			& = \mathbb{E}_{(f,p^*) \sim G_{s}}[log\frac{G_{s}(f,p^*)}{G_{s}(f,p^*) + W G_{t}(f,p^*)}] \\
			& + \mathbb{E}_{(f,p^*) \sim G_{t}}[Wlog\frac{W G_{t}(f,p^*)}{G_{s}(f,p^*) +W G_{t}(f,p^*)}] \\
			& = -log(2) + \\
			& \mathbb{E}_{(f,p^*) \sim G_{s}}[log\frac{G_{s}(f,p^*)}{(G_{s}(f,p^*) + W G_{t}(f,p^*))/2}] \ \cdots \circled{1} \\
			& + (- log(2)) + \\
			& \mathbb{E}_{(f,p^*) \sim G_{t}}[Wlog\frac{W G_{t}(f,p^*)}{(G_{s}(f,p^*) +W G_{t}(f,p^*))/2}] \ \cdots \circled{2}\\
		\end{aligned}
		\label{eq:min_G_weighted}
	\end{equation}
	Since it is easy to induce that item \circled{1} in eq. \ref{eq:min_G_weighted} equals to $D_{KL}(G_{s}(f,p^*) || \frac{G_{s}(f,p^*) + W G_{t}(f,p^*)}{2})$, here we mainly focus on proving that item \circled{2} in eq. \ref{eq:min_G_weighted} equals to $D_{KL}(W G_{t}(f,p^*) || \frac{G_{s}(f,p^*) + W G_{t}(f,p^*)}{2})$.
	\begin{proof}
		\begin{equation}
			\begin{aligned}
				& D_{KL}(W G_{t}(f,p^*) || \frac{G_{s}(f,p^*) + W G_{t}(f,p^*)}{2}) \\
				& = \int_{f}\int_{p^*} WG_{t}(f,p^*) log\frac{2 W G_{t}(f,p^*)}{G_{s}(f,p^*) +W G_{t}(f,p^*)} d_f d_{p^*} \\
				& = \int_{f}\int_{p^*} G_{t}(f,p^*) (W log\frac{2 W G_{t}(f,p^*)}{G_{s}(f,p^*) +W G_{t}(f,p^*)}) d_f d_{p^*} \\
				& = \mathbb{E}_{(f,p^*) \sim G_{t}}[Wlog\frac{W G_{t}(f,p^*)}{(G_{s}(f,p^*) +W G_{t}(f,p^*))/2}]
			\end{aligned}
		\end{equation}
	\end{proof}
	Then it is straightforward that eq. \ref{eq:min_G_weighted} equals to
	\begin{equation}
		V(G,D^{*}_{G}) = -log(4) + 2\cdot JSD(G_{s}(f,p^*)\parallel WG_{t}(f,p^*))
	\end{equation}
	When $G_{s}(f,p^*) = WG_{t}(f,p^*)$, the global minimum can be achieved as the Jensen-Shannon divergence (JSD) between two distributions is always non-negative and equals to zero iff they are exactly the same. Then in the weighted adversarial architecture, the deep neural network $G$ tends to generate equally distributed probability-feature joint outputs for the weighted target data and source data
	\begin{equation}
		\begin{aligned}
			p(G^{f}(x^s))p({F^*}^{p}(x^s)|G^{f}(x^s)) = \\
			Wp(G^{f}(x^t))p({F^*}^{p}(x^t)|G^{f}(x^t)).
		\end{aligned}
		\label{eq:osda_optimum}
	\end{equation} 
	
	\section{Appendix B. Detailed analysis of convergence curve}
	From Figure~\ref{fig:acc}, we can observe that the curves of OS(orange) rise, then go down and finally rise steadily. The curves of OS*(blue) go down and then rise. The reason is explained as follows. \textbf{1) Stage I: Low OS accuracy and high OS* accuracy.} In the first few iterations, target features and source features are far away, thus only a very small amount of target data are classified to only a few known classes (i.e., only a few known classes are taken into account when computing OS and OS*). Thus, the accuracy of the unknown class can relevantly significantly influence the accuracy of OS. Since a lot of known target data are classified as unknown at this time, the accuracy of the unknown class is small, which results in low OS accuracies. At the same time, only target data that are most close to the source data can be classified to known classes in the first few iterations, thus these target data are most likely to be correctly classified because they distribute similar to the source data. Therefore, the accuracies given by OS* are extremely high at first ($100\%$ sometimes). \textbf{2) Stage II: OS accuracies rise.} With the increase of the training iteration, target features and source features are mapped closer and more target data are correctly predicted to more known classes. Thus the influence of unknown class reduces, which results in the growth of OS. \textbf{3) Stage III: OS and OS* accuracies decline.} As the target features and source features are closer to each other along with the training process but the network is not well trained for classification task at this time, accuracies of known classes decrease, which results in decline of OS and OS*. \textbf{4) Stage IV: OS and OS* accuracies rise steadily.} Along with the increase of the training iterations, the deep network tend to be stable. More target data are classified correctly as a consequent. Thus, both OS and OS* output higher accuracies.
	
	\section{Appendix C. t-SNE embedding examples of Office-Home}
	Here we plot the t-SNE mappings of the extracted deep features in Figure~\ref{fig:tsne} for task Cl$\rightarrow$Pr and Rw$\rightarrow$Pr of the Office-Home dataset. We select to plot features of $8$ known classes (class $0\sim 7$ in Figure \ref{fig:tsne_appendix}) and the unknown class, the labels of which are annotated in Figure~\ref{fig:tsne_appendix}. The unknown class are displayed in red (class $8$ in Figure~\ref{fig:tsne_appendix}). From Figure~\ref{fig:tsne_appendix}, we can observe that the proposed AAL can map features of the target data belonging to unknown class differently from the known data for Office-Home dataset.
	\begin{figure}[t]
		\centering
		\subfigure[Cl$\rightarrow$Pr]{
			\includegraphics[width=0.47\linewidth, height=0.47\linewidth]{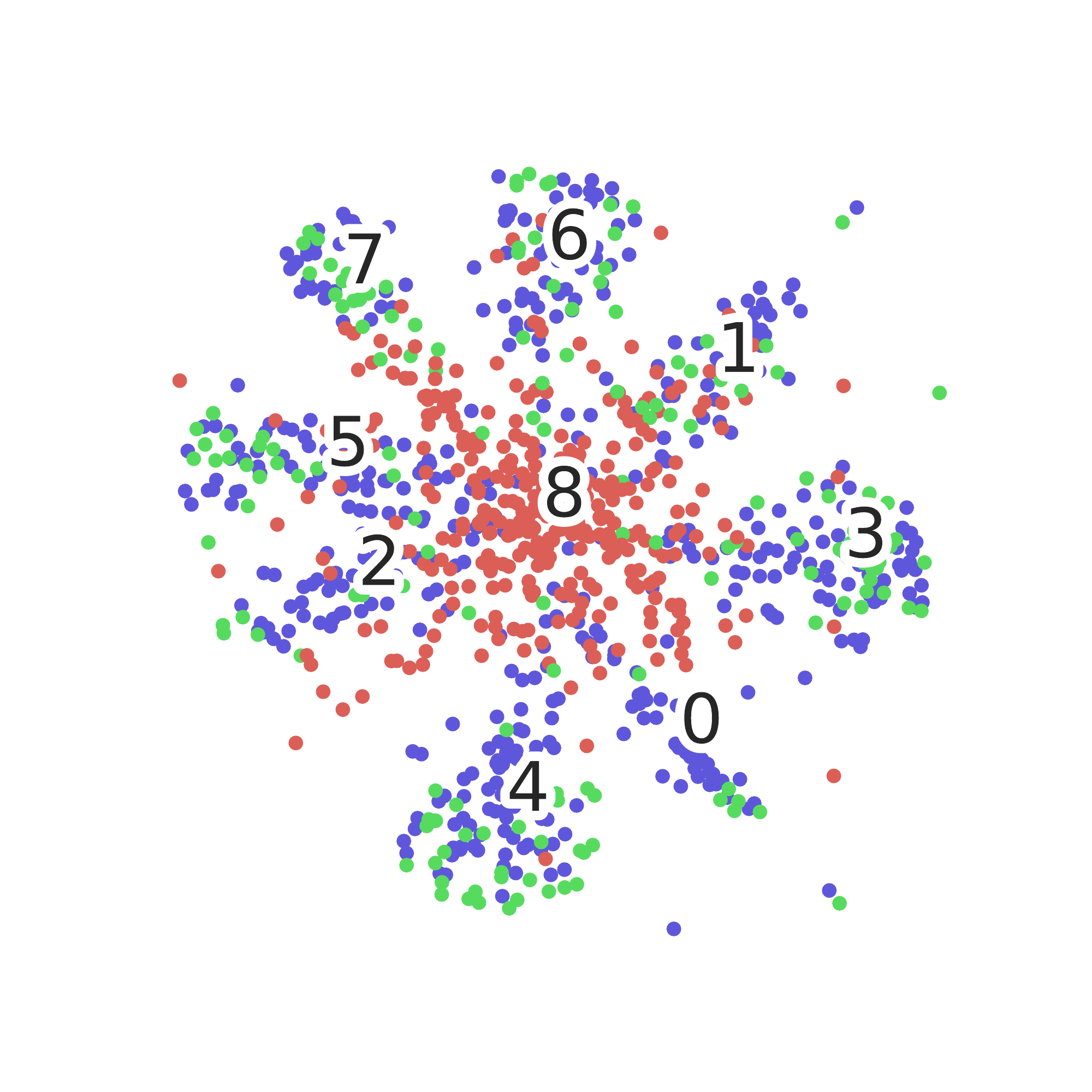}
		} 
		\subfigure[Rw$\rightarrow$Pr]{
			\includegraphics[width=0.47\linewidth, height=0.47\linewidth]{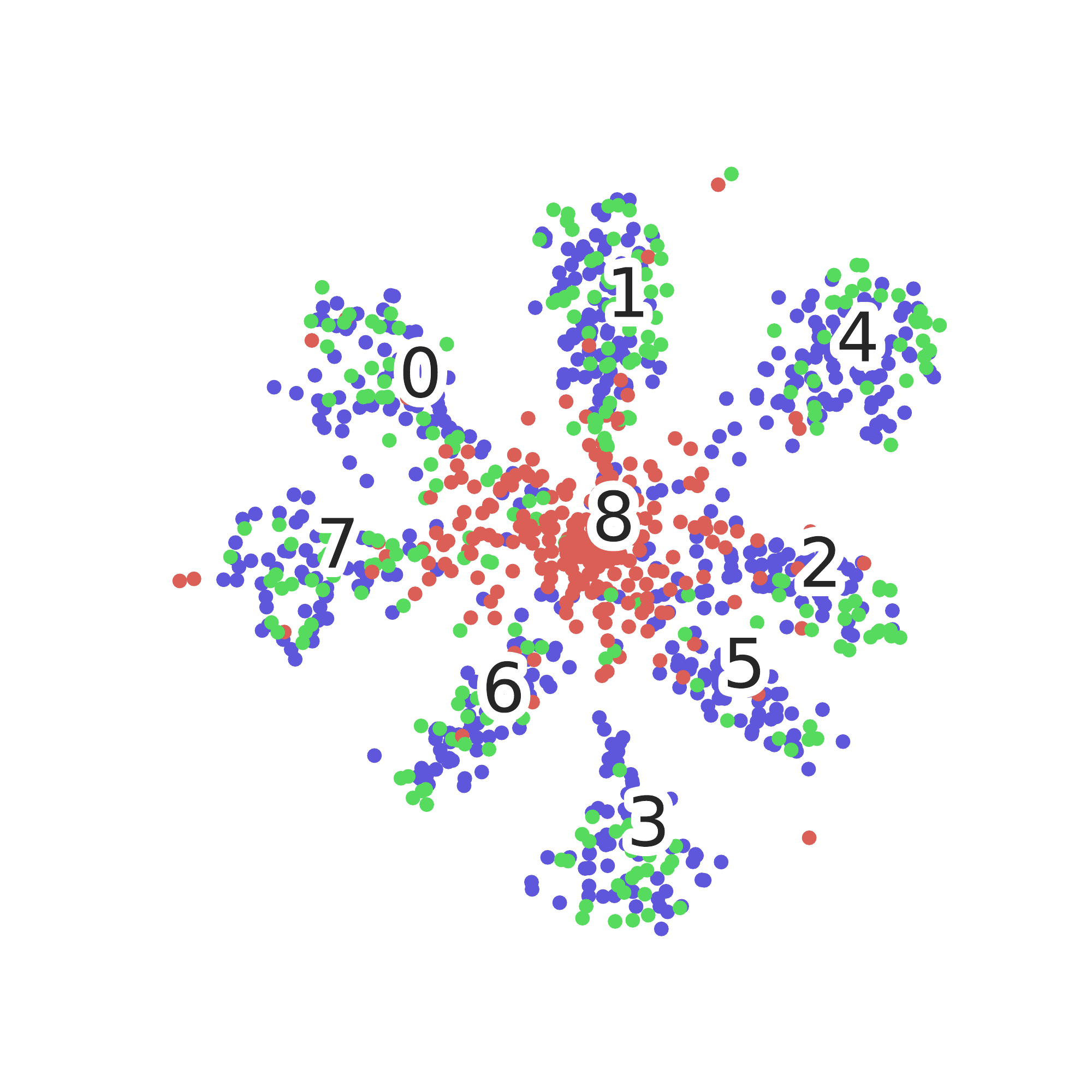}
		}
		\caption{(a), (b) are t-SNE mappings of the deep features of task $Cl\rightarrow Pr$ and $Rw\rightarrow Pr$ respectively. Known target data is denoted in green, source data is denoted in blue, and the unknown data is in red.}\label{fig:tsne_appendix}
	\end{figure}
	
	\section{Appendix D. Convergence when there are no $W$ and $\alpha$}
	In this section, the accuracies of each iteration without $W$ and $\alpha$ is plotted in Figure \ref{fig:acc_ablation}. Here task A$\rightarrow$W and D$\rightarrow$A are selected as examples. From Figure \ref{fig:acc_ablation}, we can observe that the training process is not stable and decline with the training process at last without $W$ and $\alpha$, which reveals the significance of utilizing $W$ and $\alpha$.
	\begin{figure}[h]
		\centering
		\subfigure[A$\rightarrow$W]{
			\includegraphics[width=0.85\linewidth]{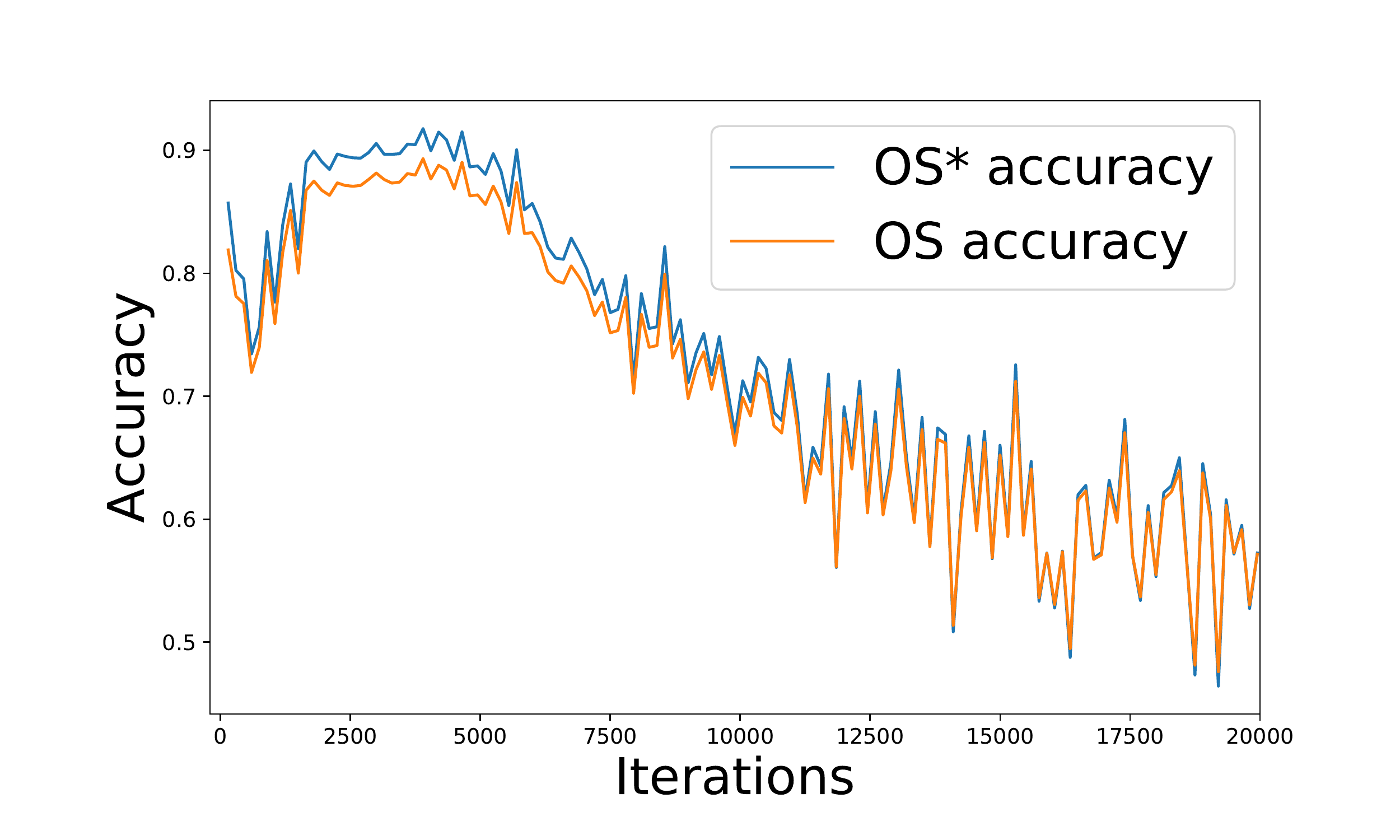}
		} \\
		\subfigure[D$\rightarrow$A]{
			\includegraphics[width=0.85\linewidth]{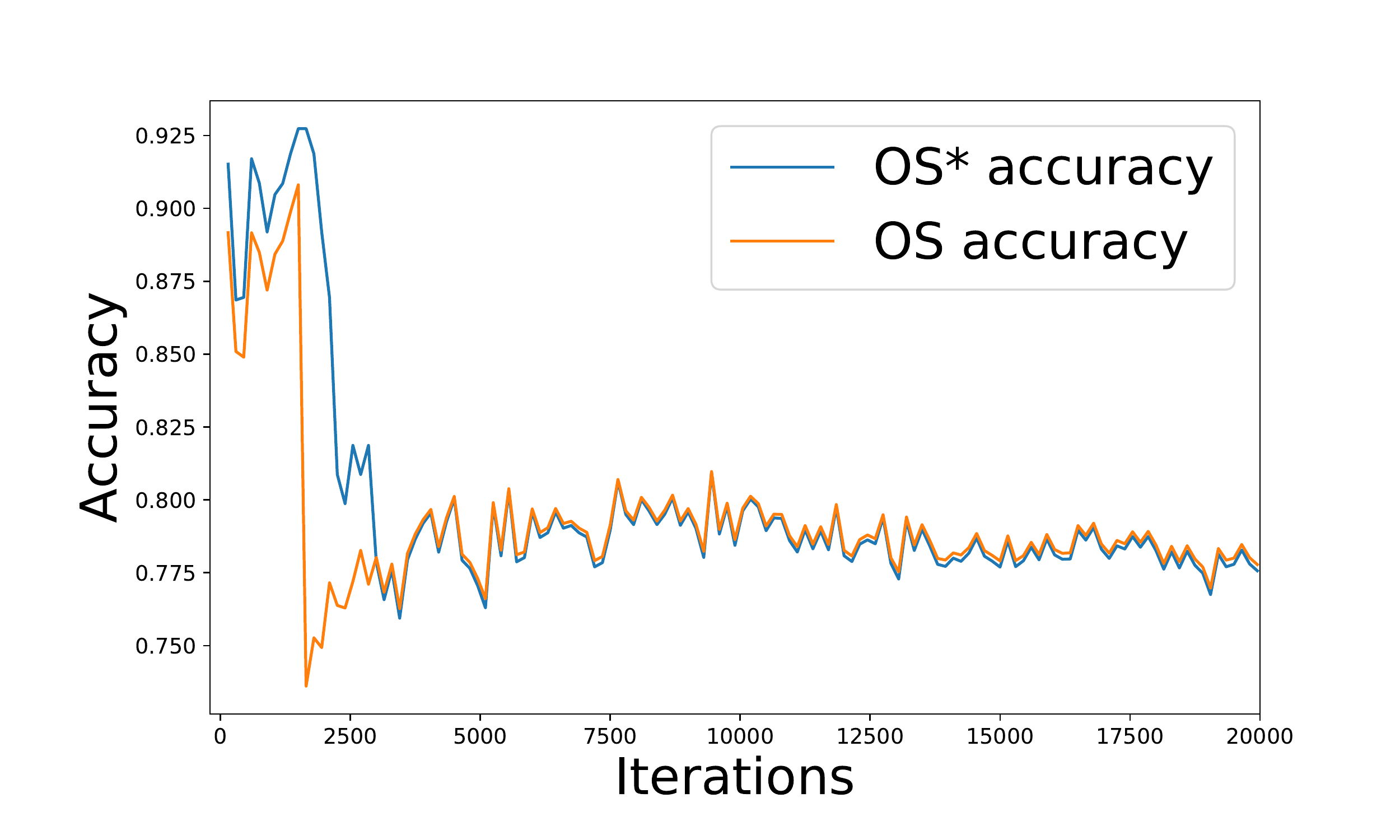}
		}
		\caption{Accuracies of each iteration for task A$\rightarrow$W and D$\rightarrow$A under the case that $W$ and $\alpha$ are not used.}
		\label{fig:acc_ablation}
	\end{figure}
\end{appendices}
\end{document}